\documentclass[10pt]{article}
\usepackage{fullpage}

\usepackage{times}  
\usepackage{url}  
\usepackage{graphicx}  

\usepackage{url}            
\usepackage{booktabs}       
\usepackage{nicefrac}       
\usepackage{microtype}      

\usepackage{graphicx}
\usepackage{color}

\usepackage{amssymb,amsmath,amsfonts}

\usepackage[algo2e, ruled, vlined]{algorithm2e}
\usepackage{algorithm}

\def\vecu{\boldsymbol{u}}
\def\vecp{\boldsymbol{p}}
\def\vecw{\boldsymbol{w}}
\def\boldpi{\boldsymbol{\pi}}

\def\boldpsi{\boldsymbol{\psi}}
\def\calA{\mathcal{A}}

\def\calM{\mathcal{M}}
\def\calQ{\mathcal{Q}}
\newcommand{\ignore}[1]{}
\newcommand{\notinproc}[1]{#1}

\def\E{{\textsf E}}

\def\eachg{\textsf{ForEach}}
\def\allg{\textsf{ForAll}}

\newtheorem{thm}{Theorem}[section]
\newtheorem{theorem}{Theorem}[section]

\newtheorem{lemma}[thm]{Lemma}

\newtheorem{corollary}[thm]{ Corollary}

\newcommand{\qed}{\hfill \rule{1ex}{1ex}\medskip\\}
\newenvironment{proof}{\paragraph{Proof}}{\qed}
\date{}
\title{Clustering Small  Samples with Quality Guarantees:\\
Adaptivity with One2all pps}
\ignore{
\author{Edith Cohen\\ Google Research, USA \\ Tel Aviv University, Israel
  \And
  Shiri Chechik\\ Tel Aviv University, Israel \And  Haim Kaplan\\ Tel Aviv University, Israel}
}
\author{Edith Cohen\\ Google Research, USA \\ Tel Aviv University, Israel
  \and
  Shiri Chechik\\ Tel Aviv University, Israel \and  Haim Kaplan\\ Tel
  Aviv University, Israel}

\begin{document}
\maketitle 

\begin{abstract}
Clustering of data points is a fundamental tool in data analysis.  We
consider points $X$ in a relaxed metric
space, where the triangle inequality holds within a  constant
factor. 
A clustering of $X$ is a partition of $X$ defined by a set of 
points $Q$ ({\em centroids}), according to the closest centroid.  The
{\em cost} of clustering $X$ by $Q$ is  $V(Q)=\sum_{x\in X} d_{xQ}$.
This formulation generalizes classic $k$-means
clustering, which uses squared distances.
Two basic tasks, parametrized by $k \geq 1$, are
{\em cost estimation}, which returns (approximate) $V(Q)$
for queries $Q$ such that $|Q|=k$  and {\em clustering}, which returns
an (approximate) minimizer of $V(Q)$ of size $|Q|=k$.
With very large data sets $X$, we seek efficient constructions of small samples that act as surrogates to the full data for performing these tasks.
 Existing constructions that provide quality guarantees are either worst-case, and unable to
benefit from structure of real data sets, or make explicit strong
assumptions on the structure.  We show here how to avoid both
these pitfalls using adaptive designs. 

 At the core of our design is the novel {\em one2all} construction of
  multi-objective probability-proportional-to-size (pps) samples:
Given a set $M$ of centroids and $\alpha \geq 1$, one2all efficiently  assigns probabilities to points
so that the clustering cost of  {\em each} $Q$ with   cost 
$V(Q) \geq V(M)/\alpha$ 
can be estimated well from a sample of size
  $O(\alpha |M|\epsilon^{-2})$.
For cost queries, we can obtain worst-case sample
size $O(k\epsilon^{-2})$ by applying
one2all to a bicriteria approximation $M$, but we adaptively balance
$|M|$ and $\alpha$ to further reduce sample size.
 For clustering, 
we design an adaptive wrapper that applies a base clustering algorithm to a
sample $S$. Our wrapper  uses the smallest sample that provides
statistical guarantees that the quality of the clustering on the
sample carries over to the full data set.  
We demonstrate experimentally the huge gains of using our 
adaptive instead of  worst-case methods.

\end{abstract}

 \section{Introduction}

Clustering is a fundamental and prevalent tool in data analysis.
We have a set $X$ of data points that lie in a
(relaxed) metric space $(\calM,d)$, where distances  satisfy a relaxed
triangle inequality: For some constant $\rho \geq 1$, for any three points $x,y,z$, 
 $d_{xy} \leq \rho(d_{xz}+d_{zy})$.  
 Note that any metric space with distances replaced by their $p$th power
satisfies this relaxation:  For $p \leq 1$ it remains a metric and
otherwise we have  $\rho = 2^{p-1}$.
 In particular, for
squared distances ($p=2$), commonly used for clustering, we have  $\rho=2$.

 Each set $Q \subset \calM$ of points ({\em centroids})
 defines a {\em clustering}, which is a partition  of $X$ into
 $|Q|$ clusters, which we denote by $X_q$ for $q\in Q$, so that a point $x\in X$ is in $X_q$ if and only if it is
 in the Voronoi region of $q$, that is $q = \arg\min_{y\in Q} d_{xy}$.
We allow points $x\in X$ to have optional weights $w_x>0$, and 
define the {\em cost} of clustering $X$ by $Q$ to be
\begin{equation} \label{cost:eq}
V(Q \mid X, \vecw) = \sum_{x\in X} w_x d_{xQ}\ ,
\end{equation}
 where $d_{xQ} = \min_{y\in Q} d_{xy}$ is the distance from point $x$
 to the set $Q$.

Two fundamental computational tasks are {\em cost queries} and {\em
  clustering}  (cost
minimization).  
The clustering cost \eqref{cost:eq} of query $Q$ can be computed using $n |Q|$ pairwise
distance computations, where $n=|X|$ is the number of points in $X$.
With multiple queries, it is useful to pre-process $X$ and return
fast approximate answers.
Clustering  amounts to finding  $Q$ of size $|Q|\leq k$
with minimum cost: 
\begin{equation} \label{ccoptimization:eq}
\arg\min_{Q \mid |Q|\leq k}   V(Q \mid X,\vecw)\ . 
\end{equation}
Optimal clustering is computationally hard
\cite{AloiseKmeansHard:ML2009} even on Euclidean spaces and even to tightly approximate
\cite{AwasthiCKS:SoCG15}.  There is a local search polynomial
algorithm with $9+\epsilon$ approximation ratio \cite{Kanungo_9appproxKmeans:2004}.
In practice, clustering
is solved using 
heuristics, most notably  Lloyd's algorithm ({\sc EM}) for squared
Euclidean distances \cite{Lloyd:1982}  and 
scalable approximation algorithms such as 
{\sc kmeans++} \cite{kmeans++:SODA2007} for general
metrics.  {\sc EM} iterates allocating points to clusters defined by the
nearest centroid, and replacing each centroid with the center of mass
$\sum_x w_x x$ of its cluster. Each iteration uses $|X|k$ pairwise
distance computations.  It is a heutistic because although each
iteration reduces the clustering
cost,  the algorithm can terminates in a local minima.  {\sc kmeans++}
produces a sequence of
points $\{m_i\}$: The first point $m_1$ is selected randomly with probability $\propto
w_x$ and a point $m_i$ us selected with
probability $\propto w_x d_{\{m_1,\ldots,m_{i-1}\}x}$. Each iteration
requires $O(|X|)$ pairwise distance computations. {\sc kmeans++}
guarantees that the expected clustering cost of the 
first $k$ points is within an $O(\log k)$ factor of the
optimum $k$-means cost.
 Moreover,  
{\sc kmeans++} provides  bi-criteria guarantees \cite{AggarwalSK:RANDOM2009,Wei:NIPS2016}:  The first 
$\beta k$ points selected (for some constant $\beta>1$) have expected
clustering cost is within a constant factor of the  optimum $k$-means
cost.   In practice,  {\em kmeans++} is often used to initiallize
Lloyd's algorithm.

 When the set of points $X$ is very large,  we seek an
 efficient method that computes a small summary structure that can
 act as a surrogate to the full data sets and allow us to efficiently 
approximate clustering costs.  These structures are
 commonly in the form of subsets  $S\subset X$
 with weights $\vecw'$ so that 
 $V(Q \mid S,\vecw')$ approximates $V(Q \mid X,\vecw)$ for each $Q$ of 
 size $k$.
Random samples are a natural form of such structures.  The challenge
is, however, that we need to choose the weights carefully:  A uniform
sample of $X$ always provide us  with unbiased estimates of
clustering costs but can miss critical points and will not provide quality guarantees.

When designing summary structures, we seek to optimize the tradeoff
between the structure size and the quality guarantees it provides.
The term {\em coresets} for such summary structures was coined in the computational geometry
literature \cite{AgarwalCoresets2005,Har-Peled:STOC2004}, building on the theory of
$\epsilon$-nets. Some notable coreset constructions include \cite{MettuPlaxton:2004,ChenCoresets:sicomp2009,FeldmanLangberg:STOC2011,Feldman:SODA2013}.
Early coresets constructions had bounds with high
 (exponential or high polynomial)  dependence on some
 parameters (dimension, $\epsilon$, $k$) and poly logarithmic dependence
 on $n$.  The state-of-the-art asymptotic  bound of $O(k\epsilon^{-2}
 \log k \log n)$ is claimed in \cite{BravermanFL:arxiv2016}.

 The bulk of coreset constructions are aimed to provide strong
 ``\allg'' statistical guarantees,  which bound the distribution of
 the maximum  approximation error  of all $Q$ of size $k$.  
The \allg\ requirement, however, comes with a hefty increase in
structure size and is an overkill for the 
two tasks we have at hand:  For clustering cost queries, weaker per-query ``\eachg'' 
typically suffice, which for each $Q$, with very high probability over 
the structure distribution, bound the error of the estimate of $V(Q)$. 
 For clustering, it suffices to  guarantee that the 
(approximate) minimizers of $V(Q \mid S,\vecw')$  are approximate 
minimizers of $V(Q \mid X,\vecw)$\notinproc{\footnote{Indeed, a notion of
``weak coresets'' aimed at only supporting
optimization, was considered in \cite{FeldmanLangberg:STOC2011}, but
in a worst-case setting}}.
Moreover, previous constructions use coreset sizes that are
{\em worst-case}, based on general  (VC) dimension or union bounds.
Even when a worst-case bound is 
tight up to constants, which typically it is not  (constants 
are not even specified in state of the art coreset constructions), it
only means it is tight for pathological 
data sets of the particular size and dimension.  A much smaller
summary structure might suffice when there is structure typical in  data 
such as natural clusterability  (which is what we
seek) and lower dimensionality than the ambient space. 

 It seems on the surface, however, that  in order to achieve 
statistical guarantees on quality of the  results 
one must either make explicit
assumptions on the data or use the worst-case size.
We show here how to avoid both these pitfalls via elegant
 adaptive designs.


\subsection*{Contribution Overview}
 Our main building block are novel summary structures for clustering costs based on
 multi-objective probability-proportional-to-size (pps) samples
 \cite{multiw:VLDB2009,multiobjective:2015}, which build on the
 classic notion  of sample coordination \cite{KishScott1971,BrEaJo:1972,Saavedra:1995,ECohen6f}.

 Consider a particular set $Q$ of centroids.  The theory of
 weighted sampling \cite{SSW92,Tille:book}  tells us that to estimate the sum $V(Q
 \mid X,\vecw)$ it suffices to sample $\epsilon^{-2}$ points with
 probabilities $p_x \propto w_x d_{Qx} $ proportional to their contribution to the sum
\cite{HansenH:1943}.  The inverse-probability
 \cite{HT52} estimate obtained from the sample $S$,  
$$\hat{V}(Q \mid X,\vecw) \equiv  V(Q \mid  S, \{w_x/p_x\})\ ,$$  is an unbiased estimate of $V(Q
 \mid X,\vecw)$ with well-concentrated (in the Bernstein-Chernoff
 sense)  normalized squared error that is
 bounded by  $\epsilon$.
The challenge here for us is that we are interested in simultaneously
having pps-like quality guarantees for {\em all} subsets $Q$ of size $k$ whereas  the estimate
$V(Q' \mid  S, \{w_x/p_x\})$ when $S$ is taken from a sample
distribution according to  $Q\not= Q'$ will not provide these guarantees for 
$Q'$.   To obtain these quality guarantees for all $Q$ by a
single sample, we  use {\em multi-objective} pps sampling
probabilities,  where the sampling probability of each point $x\in X$
is the maximum pps probability over all $Q$ of size $k$.

 Clearly, the size of a multi-objective sample will be larger than
 that of a dedicated pps sample.   Apriori, it seems that the size overhead
can be very large.  Surprisingly, we show that on any
(relaxed) metric space, the overhead is only $O(k)$. That is, a 
multi-objective pps sample of size $O(k\epsilon^{-2})$ provides, for
{\em each} $Q$ of size $k$, the same estimate quality guarantees as a
dedicated pps sample of size $O(\epsilon^{-2})$  for $Q$.
Note that the overhead  does not depend on the dimensionality of the space 
or on the size of the data.  Our result generalizes previous 
work \cite{CCK:random15} that only applied to the case where $k=1$, 
where clustering cost reduces to inverse classic closeness centrality 
(sum of distances from a single point $Q$).

 For our applications, we also need to efficiently compute these
 probabilities -- the straightforward method of enumerating over the 
 infinite number of subsets $Q$ is clearly not feasible.
 Our main technical contribution, which is the basis of both the
 existential and algorithmic results,  is an extremely simple and very
 general construction which we refer to as {\em one2all}:
Given a set $M$ of points and any $\alpha\geq 1$, we compute using 
$|M|n$ distance computation  sampling probabilities for points in $X$
that upper bound the multi-objective sampling probabilities for
all subset $Q$ with 
clustering cost $V(Q) \geq V(M)/\alpha$.  Moreover, the overhead is only
$O(\alpha |M|)$.

 By considering the one2all probabilities for an optimal clustering $M$ of size 
 $k$ and $\alpha=1$, 
we establish  {\em existentially} that a multi-objective pps sample for all 
sets $Q$ of size $k$ has size $O(k\epsilon^{-2})$.
To obtain such probabilities efficiently, we can apply {\sc kmeans++}
\cite{kmeans++:SODA2007} or another efficient bi-criteria approximation 
algorithm  to compute $M$ of size $\beta k$ (for a small constant
$\beta$)  that has cost within a factor of $\alpha>1$ than
the optimum $k$-clustering  \cite{Wei:NIPS2016}. 
We then compute one2all probabilities for $M$ and $\alpha$.

This, however,  is a worst-case construction.  We further
propose a data-adaptive enhancement that can decrease  sample size significantly
while retaining the quality guarantees:  Note that  instead of applying one2all to $(M,\alpha)$, we
can instead use $M'\subset M$ and $\alpha' \gets \alpha  V(M')/V(M)$.
Our adaptive design uses 
the sweet-spot  prefix $M'$ of the centroids sequence returned by {\sc
  kmeans++} that minimizes the sample size.

 For the task of approximate cost queries, we pre-process the data as 
 above to obtain multi-objective pps probabilities and compute a
 sample $S$ with size parameter $\epsilon^{-2}$.  
  We then process cost queries $Q$ by computing and returning  the clustering cost of 
  $S$ by $Q$:  $V(Q \mid S, \{w_x/p_x\})$.  Each computation performs 
  $O(|S| |Q|)$ pairwise distance computations instead of the $O(n 
  |Q|)$ that would have been required over the full data.  This can be 
  further reduced using   approximate nearest neighbor structures.
  Our estimate provides pps statistical 
guarantees for each $Q$ of size $k$, or more generally, for 
each $Q$ with $V(Q \mid X,\vecw) \geq  V(M \mid X,\vecw)/\alpha$.
Note that both storage and query computation are linear in the sample
size.  The worst-case sample size is
$ O(\alpha|M|\epsilon^{-2})=O(k\epsilon^{-2})$ but our adaptive design can
yield much smaller samples.

For the task of approximate clustering,  we adapt an
optimization framework over multi-objective 
samples~\cite{multiobjective:2015}.  The meta algorithm is a wrapper
that inputs multi-objective pps probabilities, specified error
guarantee $\epsilon$,  and 
a black-box (exact, approximate, bicriteria, heuristic) base clustering algorithm
$\calA$.  The wrapper applies $\calA$ to a sample to obtain a respective 
approximate minimizer of the clustering cost over the sample.  When
the sample is much smaller than the full data set, we can expect
better clustering quality using less computation.
Our initial multi-objective pps sample provides
\eachg\ guarantees that apply to each estimate in isolation but not to the sample optimum.  In particular,  it 
does not guarantee us that the
solution over the sample has the respective quality 
over the full data set.  
A larger sample may or may not be required.
One can always increase the sample by a worst-case upper bound (using
a union bound or domain-specific dimensionality arguments).   Our adaptive
approach exploits a critical benefit of \eachg: That is,
we are able to {\em test} the quality of the sample
approximate optimizer $Q$ returned by $\calA$:  If the clustering cost of $V(Q \mid X,\vecw)$ agrees with
the estimate $V(Q \mid S,\vecw')$ then we can certify that $Q$ has similar (within $(1+\epsilon)$ quality over $X$ as it has over the sample $S$.
 Otherwise,
the wrapper doubles the sample size $S$ and repeats until the test is satisfied.
Since the base algorithm is always at least 
linear, the total computation is dominated by that last largest sample 
size we use.   

Note that the only computation performed over the full data set
are the $O(k)$ iterations of {\sc kmeans++} that produce $(M,\alpha)$ to which
we apply one2all.   Each such iteration performs $O(|X|)$ distance
computations. This is a significant gain, as even with Lloyd's
algorithm ({\sc
  EM} heuristic), each iteration is $O(k|X|)$.  This design allows us
to apply more computationally intensive $\calA$ to a small sample.

A further adaptive optimization targets this initial cost:
 On real-world data it is often the case that 
much fewer iterations of  {\sc kmeans++} bring us to within some
reasonable factor $\alpha$ of the optimal $k$-clustering.   We thus
propose to adaptively perform additional {\sc kmeans++} iterations as
to balance their cost with the size of the sample that we need to work with.

 We demonstrate through experiments on both synthetic and real-world
 data the potentially huge gains of our 
 data-adaptive method as a replacement to
worst-case-bound size samples or coresets.

  The paper is organized as follows.   Pps and  multi-objective pps
  sampling in the context of clustering are reviewed in
  Section~\ref{sampling:sec}.   Section~\ref{thm:sec} presents 
our one2all probabilities and implications.
  Section~\ref{proof:sec} provides a full proof of the one2all
  Theorem.    Section~\ref{oracle:sec} present adaptive clustering
  cost oracles and Section~\ref{wrapper:sec} presents an adaptive
  wrapper for clustering on samples. 
Section~\ref{exper:sec} demonstrated experimentally, using
  natural synthetic data, the enormous gain by using data-dependent adaptive
  instead of  worst-case  sizes.  


  \section{Multi-objective pps samples for clustering} \label{sampling:sec}

We review the framework of weighted and multi-objective weighted
sampling \cite{multiobjective:2015}  in our context of clustering
costs.
Consider approximating the clustering cost $V(Q \mid X,\vecw)$
from a sample $S$ of $X$.
 For probabilities $p_x>0 $ for  $x\in X$ and a sample $S$ drawn
 according to these probabilities, we have the unbiased inverse probability estimator~\cite{HT52}  of $V(Q \mid X, \vecw)$:
 \begin{equation} \label{inverseprobest:eq}
\hat{V}(Q \mid X,\vecw)  = \sum_{x\in S} w_x \frac{d_{xQ}}{p_x} = V(Q
\mid S, \{w_x/p_x\})\ .
\end{equation}
Note that the estimate is equal to the clustering cost of $S$ with
weights $w_x/p_x$  by $Q$.

\subsection{Probability proportional to size (pps) sampling} 
To obtain guarantees on the estimate quality of the clustering cost by
$Q$, we need to use weighted sampling~\cite{HansenH:1943}.
The pps {\em base} probabilities of $Q$ for $x\in X$ are
\begin{equation} \label{basepps:eq}
\psi_x^{(Q \mid X,\vecw)} = \frac{w_x d_{xQ}}{\sum_{y\in X} w_y d_{yQ}}\ .
\end{equation}
The pps probabilities for a sample with size parameter $r>1$ are
$$r*\psi_x^{(Q \mid X,\vecw)}=\min\{1, r \psi_x^{(Q \mid X,\vecw)}\}\ .$$
Note that the (expected) sample size is $\sum_x p_x$.  When
$p_x = r*\psi_x^{(Q \mid X, \vecw)}$, the size is at most $r$.  
With pps sampling we obtain the following guarantees:
\begin{theorem} [(weak) pps sampling] \label{ppsdedicated:thm}
 Consider a sample $S$ where each $x\in X$ is included independently (or using VarOpt dependent sampling \cite{Cha82,varopt_full:CDKLT10})
with probability
$p_x \geq \alpha \epsilon^{-2}*\psi_x^{(Q \mid X,\vecw)}$, where $\alpha \leq 1$.
Then the estimate~\eqref{inverseprobest:eq} has the
following statistical guarantees:
\begin{itemize}
\item
The coefficient of
variation (CV), defined as the ratio of the standard deviation to the
mean, (measure of  the ``relative error'')  is at most
$\epsilon/\sqrt{\alpha}$.
\item
The estimate is well concentrated in the
Chernoff-Bernstein sense. In particular, we have the following bounds on the relative error:
\begin{eqnarray*}
\text{For $\delta\geq 0$,\ }\Pr[V(Q\mid S) \geq (1+\delta)V(Q 
  \mid X,\vecw)] \leq \exp(-\delta \ln(1+\delta) \alpha\epsilon^{-2}/2) \\
\text{For $\delta\leq 1$,\ }\Pr[V(Q\mid S) \leq (1-\delta)V(Q 
  \mid X,\vecw)] \leq \exp(-\delta^2\alpha\epsilon^{-2}/2) \\
\text{For $\delta\geq 1$,\ }\Pr[V(Q\mid S) \geq (1+\delta)V(Q 
    \mid X,\vecw)] \leq \frac{1}{\delta-1}\ .
\end{eqnarray*}  
\end{itemize}
\end{theorem}
\begin{proof}
  See for example \cite{multiobjective:2015}.
    To establish the CV bound, note that
the per-point contribution to the variance of the estimate is
  $(1/p_x-1)(d_{Qx} w_x)^2 \leq \alpha^{-1} \epsilon^{2} V(Q) d_{xQ}w_x$.  The sum is
  $\leq  \alpha^{-1} \epsilon^{2} V(Q)^2$ and the CV is at most $\epsilon/\sqrt{\alpha}$.
The stated confidence bounds  follow from the 
simplified multiplicative form of Chernoff bound.   The last
inequality is Markov's inequality.
\end{proof}  

For our purposes here, we will use the following bound on the probability that  with {\em
   weak pps}  sampling ($\alpha<1$) the estimate exceeds $\alpha^{-1} V(Q 
  \mid X,\vecw)$:
  \begin{corollary} [Overestimation probability] \label{weakpps:lemma}
\begin{eqnarray*}
\text{For $\alpha\leq 0.5$,}\ \Pr[V(Q\mid S) \geq \alpha^{-1} V(Q 
  \mid X,\vecw)] \leq \min\{\frac{\alpha}{1-2\alpha},
  \exp(-(1-\alpha)\ln (1/\alpha)
  \epsilon^{-2}/2)\}  
\end{eqnarray*}
\end{corollary}
\begin{proof}
   We substitute
   relative error of $\delta=(1/\alpha-1)$ in the multiplicative Chernoff bounds and by also applying Markov inequality.
\end{proof}

\subsection{Multi-objective pps}
When we seek estimates with statistical guarantees for a 
set $\calQ$ of queries (for example, all sets of $k$ points in the
  metric space $\calM$), we use multi-objective samples
  \cite{multiw:VLDB2009,multiobjective:2015}.
 The {\em multi-objective (MO) pps base sampling
  probabilities} are defined as the maximum of the pps base
probabilities over $Q\in \calQ$:
{\small
\begin{equation} \label{MObasepps:eq}
\psi_x^{(\calQ \mid X,\vecw)} = \max_{Q\in \calQ} \psi_x^{(Q \mid X,\vecw)}\ .
\end{equation}
}
Accordingly, for a size parameter $r$, the multi-objective pps
probabilities are
{\small
$$r*\psi_x^{(\calQ \mid X,\vecw)} = \min\{1,r \psi_x^{(\calQ \mid X,\vecw)}\} = \max_{Q\in\calQ} r*\psi_x^{(Q \mid X,\vecw)} \ .$$}
A key property of multi-objective pps is that the CV and
concentration bounds of dedicated (weak) pps samples (Theorem~\ref{ppsdedicated:thm} and Corollary~\ref{weakpps:lemma}) hold.
We refer to these multi-objective statistical quality guarantees as ``\eachg,'' meaning that they hold for each $Q$ over the distribution  of the samples.
We define the {\em overhead} 
of multi-objective sampling $\calQ$ or equivalently of the respective base
probabilities as:
$$h(\calQ \mid X,\vecw) \equiv |\boldpsi^{(\calQ \mid X,
\vecw)}|_1 \equiv \sum_{x\in X} \psi_x^{(\calQ \mid X,\vecw)} 
\ .$$ 
\notinproc{Note that the overhead is always between 
$1$ and $|\calQ|$.}
The overhead bounds the factor-increase in sample size due to ``multi-objectiveness:''
The multi-objective pps sample size with size parameter $r$
is at most $|r*\psi^{(Q\mid X,\vecw)}|_1\leq r h(\calQ \mid X,\vecw)$.

Sometimes we can not compute $\boldpsi$ exactly but can instead
efficiently obtain upper bounds $\boldpi \geq \boldpsi^{(\calQ \mid
  X,\vecw)}$.  Accordingly, we use sampling probabilities
$r*\boldpi$.  
The use of upper bounds increases the sample size. We refer to
$h(\boldpi) = |\boldpi|_1$ as the overhead of $\boldpi$. 
We seek upper-bounds $\boldpi$ with overhead not
much larger than $h(\calQ \mid X,\vecw)$.

\section{one2all probabilities} \label{thm:sec}

Consider a relaxed metric space $(\calM,d)$ where distances satisfy
all properties of a metric space except that the triangle inequality is relaxed using a parameter $\rho \geq 1$:
\begin{equation} \label{rhotriangle}
\forall x,y,z\in\calM,\ d_{xy} \leq \rho(d_{xz}+d_{zy})\ .
\end{equation}  

Let $(X,\vecw)$ where $X\subset \calM$ and $\vecw>0$ be weighted points in $\calM$.
For another set of points $M\subset \calM$, which we refer to as {\em
  centroids},  and $q\in M$, we denote by
   $$X^{(M)}_q = \{x\in X \mid d_{xq} = d_{xM} \}$$
 the points in $X$ that are closest to centroid  $q$. In case of ties we apply
 arbitrary tie breaking to ensure that $X^{(M)}_q$ for $q\in M$ forms a partition of $X$. 
 We will assume that $X^{(M)}_q$ is not empty for all $q\in M$, since
 otherwise, we can remove the point $q$ from $M$ without affecting the
 clustering cost of $X$ by $M$.

Our {\em one2all} construction takes one set of centroids $M$ and
computes base probabilities for $x\in X$ such that samples from it allow us 
to estimate the clustering costs of all $Q$ with estimation quality
guarantees that depends on $V(Q \mid X,\vecw)$.
For a set $M$ we define the {\em one2all base probabilities}
$\boldpi^{(M \mid X,\vecw)}$ as:
{\small
\begin{eqnarray} 
\lefteqn{\forall m\in M, \ \forall x\in
X_m,} \label{OmegaMOQk:eq}\\
 \pi^{(M|X,\vecw)}_x &=& \min\left\{1, 
\max\left\{
2\rho \frac{w_x d_{xM}}{V(M \mid X,\vecw)},
\frac{8\rho^2 w_x}{w(X_m)}
\right\}
\right\}\ .\nonumber
\end{eqnarray}
} We omit the superscripts when clear from context.

\begin{theorem} [one2all] \label{MtoMO:thm}
Consider weighted points $(X,\vecw)$ in a relaxed metric space with
parameter $\rho$,   points  $M$,  and a set $Q$ of centroids.  Then
$$\boldpi^{(M\mid X,\vecw)} \geq  \min\{1, \frac{V(Q \mid X,\vecw)}{V(M \mid X,\vecw)}\}  \boldpsi^{(Q \mid X,\vecw)}\ ,$$
where $\boldpi^{(M \mid  
  X,\vecw)}$ are the one2all base probabilities for $M$.
\end{theorem}
The full proof of the Theorem is provided in the next section.
As a corollary, we obtain that for $r\geq 1$, we can upper bound the
multi-objective base pps probabilities $\boldpsi^{(\calQ \mid
  X,\vecw)}$ and the overhead $h(\calQ)$
of the set $\calQ$ of all $Q$ with at least a fraction $1/r$ of the clustering cost of $M$:
\begin{corollary} \label{one2allplus:coro}
  Consider $M$ and $r \geq 1$ and the set
   $\calQ = \{Q \mid V(Q \mid X,\vecw) \geq V(M \mid X,\vecw)/r\}$. Then, 
  $r*\boldpi^{(M\mid X,\vecw)} \geq \boldpsi^{(\calQ \mid X,\vecw)}$
  and $h(\calQ) \leq r (8\rho^2 |M| + 2\rho)$.
\end{corollary}
\begin{proof}
For $Q\in\calQ$,   $r*\boldpi^{(M\mid X,\vecw)} \geq  r \min\{1, \frac{V(Q \mid X,\vecw)}{V(M \mid X,\vecw)}\}*\boldpsi^{(Q \mid X,\vecw)}\geq \boldpsi^{(Q \mid X,\vecw)}$.
  Note that $|\boldpi^{(M|X,\vecw)}|_1 \leq 8\rho^2 |M| + 2\rho$.
\end{proof}




We can also upper bound the multi-objective overhead of
all sets of centroids of size $k$:
\begin{corollary}
For $k\geq 1$, let $\calQ$ be the set of all k-subsets of points
in a relaxed metric space $\calM$ with parameter $\rho$.  The multi-objective  pps overhead of
$\calQ$ satisfies
$$h(\calQ) \leq 8\rho^2 k + 2\rho\ .$$
\end{corollary}
\begin{proof}
We apply Corollary~\ref{one2allplus:coro} with $M$ being the $k$-means optimum and $r=1$.
\end{proof}

\section{Proof of the one2all Theorem} \label{proof:sec}
Consider a set of points $Q$ and let 
$\alpha = \max\{1,  \frac{V(M \mid X,\vecw)}{V(Q \mid X,\vecw)} \}\ .$
To prove Theorem~\ref{MtoMO:thm},     we need to show that 
$\forall x\in X,\ $
\begin{equation} \label{claimM2}
\psi_x^{(Q\mid X,\vecw)} = \frac{w_x d_{xQ}}{V(Q\mid 
   X,\vecw)} \leq  \alpha \pi^{(M|X,\vecw)}_x\ . 
\end{equation}
We will do a case analysis, as illustrated in Figure~\ref{proof:fig}.
We first consider points $x$ such that
the distance of $x$ to $Q$ is not much larger than the distance
of $x$ to $M$.
Property \eqref{claimM2} follows using the first term
of the maximum in \eqref{OmegaMOQk:eq}.
\begin{lemma} \label{closer2Q:lemma}
Let $x$ be such that $d_{xQ} \leq 2\rho d_{xM}$.  
Then
$$\frac{w_x d_{xQ}}{V(Q \mid X,\vecw)}  \leq  2\rho \alpha  \frac{w_x d_{xM}}{V(M \mid X,\vecw)}
\ .$$
\end{lemma}
\begin{proof}
Using $V(Q \mid X,\vecw)  \geq  V(M \mid X,\vecw)/\alpha$ we get
$$\frac{d_{xQ}}{V(Q \mid X,\vecw)} \leq  \alpha \frac{d_{xQ}}{V(M
  \mid X,\vecw)} \leq 2\rho \alpha  \frac{d_{xM}}{V(M
\mid X,\vecw)} \ .$$
\end{proof}

\begin{figure}
\center 
\includegraphics[width=0.45\textwidth]{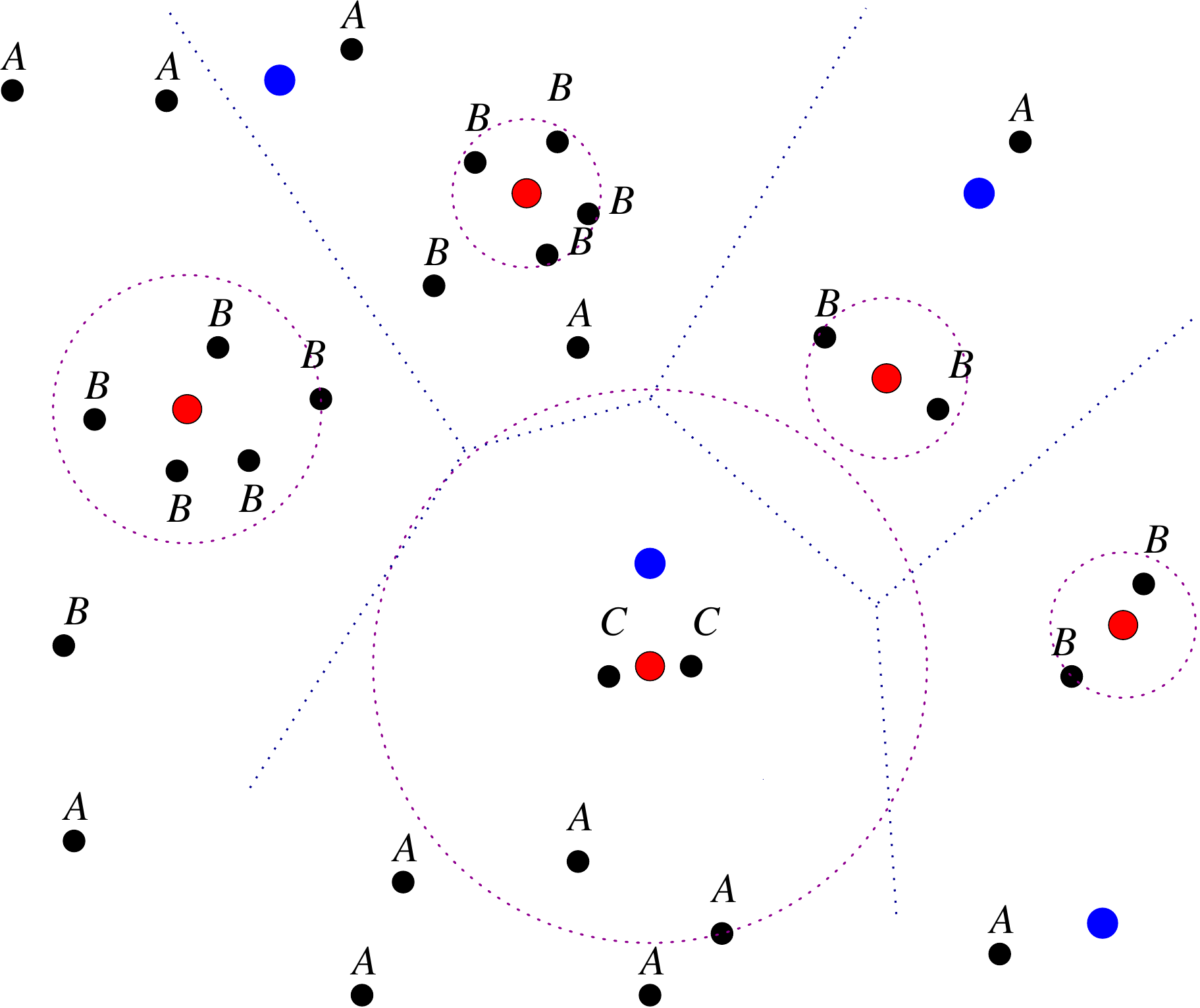}
\caption{Illustration of the one2all construction proof with $\rho=1$.  The data points $X$ are in 
    black.  The points in $M$ are colored red.  We show the respective
    Voronoi partition and for each cluster, we show circles centered at the respective 
    $m\in M$ (red) point with radius $\Delta_m$.  The points in blue are a 
    set $Q$.   The points $x\in X$ are labeled $A$ if $d_{xQ}< 2 
    d_{xM}$ (and we apply Lemma~\ref{closer2Q:lemma}).  Otherwise, when there is 
    a point $m$ such that $d_{xQ}> d_{xm}$, the point is labeled $B$
    when $d_{mQ} \geq 2\Delta_m$ (Lemma~\ref{condB:lemma}) and is 
    labeled $C$ otherwise (Lemma~\ref{condC:lemma}).}
\label{proof:fig}
\end{figure}

It remains to consider the complementary case where point $x$
 is much closer to $M$ than to $Q$:  
\begin{equation} \label{caseBC}
d_{xQ} \geq 2\rho d_{xM}\ .
\end{equation}
We first introduce a useful definition:  For a 
point $q\in M$,  we denote 
by $\Delta_q$ the weighted median of the distances $d_{qy}$ for $y\in
X_q$, weighted by $w_y$.  The median $\Delta_q$ is a value that 
satisfies the following two conditions: 
\begin{eqnarray}
\sum_{x\in X_q \mid d_{xq}\leq \Delta_q} w_x &
  \geq& \frac{1}{2} w(X_q)  \label{defmeds:eq}\\
  \sum_{x\in X_q  \mid d_{xq}\geq \Delta_q} w_x &\geq&
                                                      \frac{1}{2}  w(X_m) \label{defmedl:eq} \ . 
\end{eqnarray}  
It follows from \eqref{defmedl:eq} that for all $q\in M$, 
{\small
\begin{eqnarray*}
V(M \mid X_q,\vecw) &=&  \sum_{x\in X_q} w_x d_{qx} \geq \sum_{x\in X_q
  \mid d_{xq}\geq \Delta_q} w_x d_{xq} \\
&&\geq \Delta_q \sum_{x\in X_q
  \mid d_{xq}\geq \Delta_q} w_x \geq \frac{1}{2} w(X_m)\Delta_q\ .
\end{eqnarray*}
}
Therefore, 
{\small
\begin{equation} \label{allmeds:eq}
V(M \mid X,\vecw) = \sum_{q\in M}  V(M \mid X_q,\vecw) \geq \frac{1}{2} \sum_{q\in M}
w(X_m)\Delta_q\ . 
\end{equation}
}

We now return to our proof for $x$ that satisfies \eqref{caseBC}.  We will show that property \eqref{claimM2} holds using the second term
in the $\max$ operation in the definition \eqref{OmegaMOQk:eq}.  
Specifically, let $m$ be the closest $M$  point to $x$.  We will show
that 
\begin{equation} \label{lclaim}
\frac{d_{xQ}}{V(Q \mid X,\vecw)} \leq 8\rho^2 \alpha \frac{1}{w(X_m)}\ .
\end{equation}

We divide the proof to two subcases, in the two following Lemmas, each covering the complement of the
other:  When $d_{mQ} \geq 2\rho \Delta_m$ and when $d_{mQ} \leq 2 \rho
\Delta_m$.

\begin{lemma} \label{condB:lemma}
Let $x$ be such that 
\begin{equation*}
\exists m\in M,\  d_{mx} < \frac{1}{2\rho} d_{xQ} \text{ and } d_{mQ} \geq 2\rho \Delta_m \ .
\end{equation*}
Then 
$$\frac{d_{xQ}}{V(Q \mid X,\vecw)} \leq \frac{8\rho^2}{w(X_m)}\ .$$
\end{lemma}
\begin{proof}
Let $q = \arg\min_{z\in Q} d_{mz}$ be the closest $Q$ point to $m$.
From (relaxed) triangle inequality \eqref{rhotriangle} and our assumptions:
$$d_{xQ} \leq d_{xq} \leq \rho(d_{mq} + d_{mx}) = \rho (d_{mQ}
+d_{mx}) \leq \rho d_{mQ}+ \frac{1}{2} d_{xQ}\ .$$ 
Rearranging, we get
\begin{equation} \label{c2p1}
d_{xQ} \leq 2\rho d_{mQ}\ .
\end{equation}
Consider a point $y$ such that $d_{my} \leq \Delta_m$.  Let $q' =
\arg\min_{z\in Q} d_{yz}$ be the closest $Q$ point to $y$.
From relaxed triangle inequality we have  $d_{mq'} \leq \rho(d_{yq'} +
d_{ym})$ and therefore
{\small
\begin{eqnarray*}
d_{yQ} &=& d_{yq'} \geq \frac{1}{\rho} d_{mq'}-d_{ym} \geq \frac{1}{\rho} d_{mQ}-
\Delta_m\\ &\geq& \frac{1}{\rho} d_{mQ}-  \frac{1}{2\rho} d_{mQ}  \geq
  \frac{1}{2\rho}  d_{mQ}\ .
\end{eqnarray*}
}
Thus, using the definition of $\Delta_m$ \eqref{defmeds:eq}:
{\small
\begin{eqnarray} 
V(Q \mid X,\vecw) &\geq& \sum_{y\mid d_{yQ}\leq \Delta_m} w_y d_{yQ} \geq
\frac{1}{2 \rho} \sum_{y\mid d_{yQ}\leq \Delta_m} w_y d_{m Q}  \nonumber\\
&\geq& 
\frac{1}{2 \rho} d_{mQ} \sum_{y\in X_m \mid d_{yQ}\leq \Delta_m} w_y  \nonumber\\
 &\geq& \frac{1}{2\rho} d_{mQ}
\frac{w(X_m)}{2} = \frac{1}{4\rho} d_{mQ} w(X_m) \ . \label{c2p2}
\end{eqnarray}
}
 Combining \eqref{c2p1} and \eqref{c2p2} we obtain:
$$\frac{d_{xQ}}{V(Q \mid X,\vecw)} \leq \frac{2\rho d_{mQ}}{\frac{1}{4\rho}w(X_m) d_{mQ}}
= 8\rho^2 \frac{1}{w(X_m)} \ .$$

\end{proof}

\begin{lemma}  \label{condC:lemma}
Let a point $x$ be such that 
\begin{equation*}
 \exists m\in M,\ 
d_{xm} < \frac{1}{2\rho} d_{xQ} \text{ and }  d_{mQ} \leq 2\rho\Delta_m\ .
\end{equation*}
 Then 
$$\frac{d_{xQ}}{V(Q \mid X,\vecw)} \leq  8\rho^2\alpha \frac{1}{w(X_m)}\ .$$
\end{lemma}
\begin{proof}
Let $q = \arg\min_{z\in Q} d_{zm}$ be the closest $Q$ point to $m$.  We have 
$$d_{xQ} \leq d_{xq} \leq \rho(d_{xm}+ d_{mq}) \leq \frac{1}{2} d_{xQ} + \rho d_{mQ}
\leq \frac{1}{2} d_{xQ} + 2 \rho^2 \Delta_m$$
Therefore,
\begin{equation} \label{c3p1}
d_{xQ} \leq 4 \rho^2 \Delta_m\ .
\end{equation}
Using \eqref{allmeds:eq} we obtain
\begin{eqnarray} 
V(Q \mid X,\vecw) &\geq&    V(M \mid X,\vecw)/\alpha \geq  \frac{1}{2\alpha}
\sum_{y\in M} w(X_y) \Delta_y \nonumber \\
&\geq&    \frac{1}{2\alpha} w(X_m) \Delta_m\ . \label{c3p2}
\end{eqnarray}

  Combining \eqref{c3p1} and \eqref{c3p2} we obtain
$$\frac{d_{xQ}}{V(Q \mid X,\vecw)} \leq \frac{4\rho^2 \Delta_m}{\frac{1}{2\alpha}
   w(X_m) \Delta_m} \leq 8\rho^2\alpha
\frac{1}{w(X_m)}\ .$$
\end{proof}

\begin{table*}[!ht]
{\tiny 
\begin{tabular}{rrr r ||  l l  r ||l cc || r}
{\color{blue} $n$} &{\color{blue} $d$} & {\color{blue}$k$} & guarantee {\color{blue} $\epsilon$}  & adaptive 
                                                        {\color{blue}
                                                        $\frac{|S|}{n}$}
  & worst-case {\color{blue} $\frac{|S|}{n}$} & {\color{blue} $\times$ 
  gain} & {\tiny {\color{blue} est. err}}
 &    {\scriptsize {\color{blue} $\frac{V(Q \mid   X)}{V_{\text{ground-truth}}}$}}
  &   {\tiny {\color{blue} $\frac{V(\{m_0,\ldots,m_k\}  \mid
    X)}{V_{\text{ground-truth}}}$}} & sweet-spot
  \\
\hline
\multicolumn{10}{c}{Mixture of Gaussians data sets}\\
\hline 
$5\times10^5$ & $10$ & $5$ &   $ 0.10$ & $0.0500$ & $1.00$ &
                                                              {\color{red}
                                                              20.0} &
                                                                      $0.008$
 &  $1.07$ & $2.50$ & $2.3$ \\
$5\times10^5$ & $10$ & $5$  &   $0.20$ & $0.0136$ & $1.00$ &{\color{red} 73.3} &  $0.012$ &   $1.10$  & $2.39$ & $2.6$ \\

$2.5 \times 10^6$ & $10$ & $5$  &   $0.20$ & $0.0025$ & $1.00$ &
                                                                   {\color{red} 90.2} &  $0.0160$ &  $1.14$ & $2.14$ & $2.2$ \\
$1\times 10^7$ & $10$ & $5$  &   $0.20$ & $0.00066$ & $1.00$ &
                                                                  {\color{red}
                                                                  94.4}
        & $0.018$ &  $1.12$ & $2.07$ & $2.6$ \\

\hline
$2 \times 10^6$  & $10$ & $20$ &   $0.10$  &  $0.04839$ & $1.00$ &
                                                                     {\color{red} 20.7} &  $0.0018$ &
                                                                  $1.14$ 
  & $2.27$ & $9.1$ \\
$2 \times 10^6$  & $10$ & $20$  &   $0.20$ & $0.012007$ & $1.00$ &
                                                             {\color{red} 83.2} &  $0.008$ &
                                                                     $1.18$
  & $2.25$ & $9.0$ \\
$2 \times 10^6$  & $10$ & $50$ &   $0.20$ & $0.0298$ & $1.00$ &
                                                                  {\color{red} 33.5} &
                                                                $0.0057$
 &  $1.16$ & $2.24$ & $19.5$ \\
        $2\times 10^6$ &$10$& $100$ &   $0.20$ & $0.061918$ & $1.00$ &
                                                                       {\color{red} $16.2$} &  $0.0058$ &  $1.15$ & $2.22$ & $40.8$ \\
\hline

$1\times 10^6$ & $20$ & $10$ &   $0.10$ & $0.05293$ & $1.00$ & {\color{red} $18.9$} &  $0.0035$ &
                                                                   $1.17$
  & $2.39$ & $5.0$ \\

$1\times 10^6$  &$50$& $10$ &   $0.10$ & $0.04726$  & $1.00$ &
                                                               {\color{red}  $21.2$} & $ 0.0037$ &
                                                                  $1.19$
                                                                                                        &
                                                                                                          $2.65$ & $3.9$ \\

$1\times 10^6$ & $100$ & $10$ &   $0.10$ & $0.05287$ & $1.00$ &
                                                                {\color{red} $18.9$} &  $0.0035$ &
                                                                   $1.18$
  & $2.6$ & $4.9$ \\
\hline                                                                  
\multicolumn{10}{c}{MNIST data set}\\
\hline 
$6\times 10^5$  & $784$ & $10$  &   $0.20$  & $0.0371$ & $1.00$ &
                                                                    {\color{red}
                                                                    26.9}
        &  $0.018$ &  $0.985$ &  $1.765$ & $1.0$ \\
        \hline                                                                  
\multicolumn{10}{c}{Fashion data set}\\
\hline
$6\times 10^5$ & $784$ & $10$ &   $0.20$ & $0.05720$ & $1.00$ & {\color{red} 17.5}  &  $0.021$ &  $0.91$ & $1.65$ & $1.0$
\end{tabular}\caption{Clustering over a sample.  The reported
  estimation error 
  is {\small $\sqrt{A[\left(\frac{V(Q\mid X) -V(Q\mid S)}{V(Q\mid
        X)}\right)^2]}$} \label{results:tab}.}
} 
\end{table*}

\begin{algorithm}[h]\caption{Clustering cost oracle \label{oracle:alg}}
\DontPrintSemicolon 
    {\small
      \tcp{{\bf Preprocessing}}
   \KwIn{points $X$, weights $\vecw>\boldsymbol{0}$, iteration limit $\ell$, $C>0$,  $\epsilon>0$}
   \KwOut{Sample $S$ with weights $\vecw'_x$ for $x\in S$}
$s \gets |X|$; $M\gets \perp$ \tcp*{Initialization}
\tcp{Apply  {\sc kmeans++} to $(X, \vecw)$ and compute $\vecp$}
\ForEach{iteration $i\in[\ell]$ of {\sc kmeans++}$(X,\vecw)$}
{
$m_i \gets $ new centroid selected ; $M\gets M \cup \{m_i\}$\;\tcp*{Centroids}
$v_i \gets V(M \mid X,\vecw)$\tcp*{Clustering cost}
\ForEach(\tcp*[h]{one2all probabilities \eqref{OmegaMOQk:eq}}){$x\in X$}{$p_x' \gets \min\{1, \max\{1,\frac{v_i}{C}\}\epsilon^{-2} \pi_x^{(M\mid  X,\vecw)}\}$}
\lIf(\tcp*[h]{best so far}){$|\vecp'|_1 < s$}{$\vecp \gets \vecp'$}
}
 \tcp{Compute sample from $\vecp$}
 $S  \gets $ include each $x\in X$ with probability  $p_x$
\; \tcp*{Poisson or varopt sample}\;
\lForEach(\tcp*[h]{inverse probability weights}){$x\in S$}{$w'_x \gets  w_x/p_x$}
\Return $(S,\vecw)$ \tcp*{weighted sample}

\tcp{{\bf Oracle}}
\KwIn{$Q$ such that $|Q|=k$; weighted sample $(S,w'_x)$}
\KwOut{Estimate of $V(Q \mid X,\vecw)$}
\Return{$V(Q \mid S,\vecw')$}

}
 \end{algorithm}

\begin{algorithm}[h]\caption{Oracle with feedback for clustering cost \label{feedbackoracle:alg}}
\DontPrintSemicolon 
    {\small
      \tcp{{\bf Initialization}}
 \Indp
  
 \KwIn{Points $X$, weights $\vecw>\boldsymbol{0}$, $k$ , $\epsilon>0$}
 $\vecp, C \gets $ Probabilities $\vecp$ and cost $V(M_{2k})$ computed by Algorithm~\ref{oracle:alg} with $\ell=2k$, $C=V(M_{2k})$\;\\
\lForEach(\tcp*[h]{Randomization for sampling}){$x\in X$}{$u_x \sim  U[0,1]$}
$S \gets \{ x \mid u_x \leq p_x \}$\tcp*[h]{Compute the sample}\;\\
\lForEach(\tcp*[h]{weights for sampled points}){$x \in S$}{$w'_x \gets w_x/p_x$}
\Indm
\tcp{{\bf Query processing with feedback}}
\Indp
\KwIn{$Q$ such that $|Q|=k$}
\KwOut{Estimate of $V(Q \mid X,\vecw)$}
$\hat{V} \gets V(Q \mid S,\vecw')$\tcp*{Sample-based estimate of $V(Q \mid X,\vecw)$}\;
\lIf(\tcp*[h]{Return estimate and break}){$\hat{V} > C$}{\Return $\hat{V}$}
\tcp{Increase sample size}
$V \gets V(Q\mid X,\vecw)$ \tcp*{Can also use estimate}
$\vecp \gets (2C/V)*\vecp$ \tcp*[h]{increase sampling probabilities}\;\\
$C \gets V/2$ \tcp*[h]{New cost threshold}\;\\
$S \gets \{ x \mid u_x \leq p_x \}$\tcp*[h]{Update the sample}\;\\
\lForEach(\tcp*[h]{update weights}){$x \in S$}{$w'_x \gets w_x/p_x$}\;
\Return{$V$}\;

\Indm
}
\end{algorithm}

\section{Clustering cost oracle} \label{oracle:sec}

A clustering cost oracle preprocesses the data $(X,\vecw)$ and
computes a compact structure from which clustering cost queries $Q$
can be efficiently approximated.  Our basic oracle, Algorithm~\ref{oracle:alg},
inputs the data, iteration limit $\ell\geq 1$, $C>0$, and
$\epsilon>0$.  We will establish the following
\begin{theorem}
Algorithm~\ref{oracle:alg}  computes a weighted sample $(S,\vecw')$.
The inverse probability estimator \eqref{inverseprobest:eq}  provides the pps quality guarantees of Theorem~\ref{ppsdedicated:thm} and Corollary~\ref{weakpps:lemma} for sets $Q$ with clustering cost at least $\alpha C$, where $\alpha\leq 1$.
\end{theorem}
\begin{proof}
The algorithm computes probabilities $\boldpi:X$ that upper bound the
base pps probabilities for all sets $Q$ of cost $V(Q)\geq
C$.  We perform $\ell$ iterations of {\sc kmeans++} on $(X,\vecw)$.  
Each iteration $i$ computes a new centroid  $m_i$ and we can also
compute one2all base probabilities 
$\boldpi^{(M_i)}$, where  $M_i=\{m_1,\ldots,m_i\}$ for the set of the first $i$ centroids. 
Note that the computation of the
one2all probabilities and the
cost $V(M_i \mid X,\vecw)$ does utilizes the distance computations and the assignment of points to the nearest centroid that is
already performed by {\sc 
  kmeans++}.
Each iteration $i$ yields {\em candidate} base
probabilities 
$\boldpi \gets \max\{1,  \frac{V(M_i)}{C}\}* \boldpi^{(M_i)}$.
From Corollary~\ref{one2allplus:coro}, each candidate base probabilities 
upper bound the base pps probabilities of all $Q$ with cost $V(Q)\geq
C$.
Finally, we retain, among the $\ell$ candidates, the one with minimum
sample size  $|\epsilon^{-2}*\boldpi|_1$.
This sweet-spot search replaces simply using $M_\ell$.  Note that
the size of the sample may increase with $i$, when 
$V(M_i)$ drops slower than the increase in $i$.  
In our experiments 
we demonstrate the potential significant benefits of this adaptive optimization.

 We proceed and compute a weighted sample $S$ (independent or varopt) according to
 probabilities $\vecp \gets \epsilon^{-2}*\boldpi$.  For each $x\in S$
 we associate a weight $w'_x = w_x/p_x$.  
We process queries $Q$ by computing 
the clustering cost $V(Q \mid (S,\vecw')$, which is equal to the 
inverse probability estimator \eqref{inverseprobest:eq} of the 
clustering cost of $Q$ over $(X,\vecw)$. 

Since $\vecp$ upper bound pps sampling probabilities for any $Q$ with
$V(Q) \geq C$ and are within $V(Q)/C$ of the pps probabilities for any
$Q$, the quality guarantees of Theorem~\ref{ppsdedicated:thm} and
Lemma~\ref{weakpps:lemma} follow.
\end{proof}

Finally, note that the size of the oracle structure and the computation of each query 
are both linear in our sample size $|\vecp|_1$. 
The sample size we obtain using $M_i$ is (in expectation)
\begin{equation} \label{sizeexact:eq}
|\vecp|_1 = \min_{i\in[\ell]} \max\{1,\frac{V(M_i)}{C}\}
\epsilon^{-2}*\boldpi^{(M_i)} \ .
\end{equation}  
A useful simple rough approximation for sample size that does not use
the size parameter is
\begin{equation} \label{sizeapprox:eq}
(8 \rho^2
  |M|+2\rho)\frac{V(M_i)}{C}\epsilon^{-2} = O(i \frac{V(M_i)}{C}) \propto i V(M_i)\ .
\end{equation}

\subsection{Feedback oracle}
 We consider here constructing an oracle that provides quality
 guarantees for all $Q$ of size $k$.   

Assume first that we are provided with $V^*$, which is the optimal clustering
 cost with $k$ clusters.   
\begin{lemma}
Applying
 Algorithm~\ref{oracle:alg} with $\ell=2k$ and $C = V^*$ provides us
 with a sample of expected size $O(k\epsilon^{-2})$  that provides the
 statistical guarantees of Theorem~\ref{ppsdedicated:thm} (with $\alpha=1$).
\end{lemma}
\begin{proof}
 From state-of-the-art bi-criteria bounds~\cite{Wei:NIPS2016}, we have that 
$\E[V(M_{2k} \mid X,\vecw)]/V^*]  = O(1)$.  The expected size of 
the sample, even with one2all applied to $M_{2k}$, is at most $2 
\frac{V(M_{2k})}{C} k \epsilon^{-2} 8\rho^2 = O(k\epsilon^{-2})$.

We comment that we can apply sweet-spot selection of $M_i$ even though
$V(M_{2k})$ is not known before 
iteration $2k$, by using the rough approximation \eqref{sizeapprox:eq} 
instead of exact sample sizes.  This allows for 
retaining one candidate $\boldpi$ with the {\sc kmeans++}
iterations.
\end{proof}

Note, however, that we do not know the optimal 
  clustering cost $V^*$.  One solution is to underestimate it:
From the bi-criteria bounds we can compute large enough
$\alpha$ (using Markov inequality) so that  within  the desired
confidence value,
$V(M_{2k}\mid X)/V^*\leq \alpha$.
We can then apply the algorithm with $C=V(M_{2k})/\alpha$.  But such  a
worst-case $\alpha$ is large (see Section~\ref{exper:sec}) and forces 
a proportional increase in sample size, often needlessly so.

We instead propose a {\em feedback} oracle, detailed in Algorithm~\ref{feedbackoracle:alg}.  We initialize with the
basic oracle (Algorithm~\ref{oracle:sec}) with $\ell=2k$ and
$C=V(M_{2k})$ to obtain probabilities $\vecp$.  We draw a weighted
sample $(S,\vecw')$.  The oracle processes a query $Q$ as follows.  If $V(Q \mid S,\vecw') \geq C$, it returns this estimate.
Otherwise,  we compute and return the exact cost $V(Q \mid X,\vecw)$
and update the sample at the base of the oracle so that it supports
queries with cost $\geq C= V(Q \mid X,\vecw)/2$.

Note that each oracle call that results in an update halves
(at least) the cost threshold $C$.  Since we start with $C$ that is in expectation within a constant factor $\alpha$ from the optimal $k$ clustering cost,
the expected total number of oracle calls that result in an update, is
$\leq \log_2 \alpha$.
Moreover, the sample size is increased only in the face of evidence of
a clustering with lower cost.   So the final size uses $C \geq V^*/2$.
For smoother estimates as the samples size increases,  we coordinate
the samples by using the same randomization $\vecu$.
That way, new points are added to
 the sample when the size increases, but no points are removed.  

  Our feedback oracle provides the following statistical guarantees on
  estimate quality.
\begin{theorem}
A query $Q$ processed when $V(Q \mid X,\vecw)\geq \alpha C$ ($\alpha
\leq 1$) has the statistical
guarantees as stated in Theorem~\ref{ppsdedicated:thm}.   When $V(Q
\mid X,\vecw) = \alpha C$ for ($\alpha<1$),  we either return an 
exact value or an 
overestimate with probability bounded by Corollary~\ref{weakpps:lemma}.
\end{theorem}


\begin{algorithm}[h]\caption{Clustering Wrapper  \label{optsamples:alg}}
 \DontPrintSemicolon 
 {\small 
  \KwIn{points $X$, weights $\vecw>\boldsymbol{0}$, 
   $\epsilon>0$, a clustering algorithm $\calA$ that inputs a 
    weighted set of points and returns $Q\in \calQ$.}
\KwOut{Set $Q$ of $k$ centroids with statistical guarantees on quality 
  over $X$ that match within $(1+\epsilon)$  those provided by $\calA$}
 \tcp{Initialization}
$c \gets \infty$
\tcp{Apply  {\sc kmeans++} to $(X, \vecw)$}
\ForEach{iteration $i\in[2k]$ of {\sc kmeans++}$(X,\vecw)$}
{
$m_i \gets $ next centroid\;\\
$v_i \gets V(\{m_1,\ldots,m_i\} \mid X,\vecw)$\tcp*{Clustering cost}
\If(\tcp*[h]{sweet-spot one2all prob. \eqref{OmegaMOQk:eq}}){$i v_i < c$}{$\boldpi \gets \boldpi^{(M\mid X,\vecw)}$;\; $V_M 
  \gets v_i$ ;\; $c \gets i v_i$}
}
$r \gets \frac{v_M}{v_{2k}}$\tcp*{Initial sample size increase 
  factor}\; 
$Q^* \gets \{m_1,\ldots,m_k\}$; $\overline{V}\gets v_k$\tcp*{Best so far and upper bound}
\lForEach(\tcp*[h]{Randomization for sampling}){$x\in X$}{$u_x \sim  U[0,1]$}
$S \gets \{ x \mid u_x \leq r \epsilon^{-2}\pi_x \}$  \tcp*[h]{Initial sample. $O(|S|)$ given 
    preprocessed $\boldpi$}\;\\
\lForEach(\tcp*[h]{weights for sampled points}){$x \in S$}{$w'_x \gets w_x/\min\{1,r \pi_x\}$}

\tcp{Main Loop}
\Repeat{{\bf True}
}{
  \tcp{Cluster the sample $S$}

  $Q \gets \calA(S,\vecw')$\tcp*{Apply algorithm $\calA$ to sample}
  $V_Q \gets  V(Q\mid X,\vecw)$  \tcp*{Exact or approx using a validation sample}
\lIf{$V_Q < \overline{V}$}{$\overline{V}\gets V_Q$\; $Q^* \gets Q$}
\If{$V_Q \leq (1+\epsilon) V(Q \mid S, \vecw')$
  {\bf and }
  $V_Q \geq V_M/r$}{{\bf break}}
$r \gets \max\{2,  V_Q/V_M \} r$ \tcp*{Increase the sample size parameter}
\Repeat(\tcp*[h]{Increase sample size until $Q$ is cleared}){$V(Q\mid S,\vecw') > \min\{(1+\epsilon) \overline{V}, (1-\epsilon)V_Q$\}}
{$S \gets \{ x \mid u_x \leq r \epsilon^{-2}\pi_x \}$  \tcp*[h]{Add points to sample}\;\\
  \lForEach(\tcp*[h]{weights for sampled points}){$x \in S$}{$w'_x \gets w_x/\min\{1,r \pi_x\}$}
  $r \gets 2r$
  }
}
\Return{$Q^*$}
 }
\end{algorithm}

\section{Clustering  wrapper} \label{wrapper:sec}
The input to a clustering problem is  a (weighted)
set of points $(X,\vecw)$ and $k>0$.  The goal is to compute a 
set $Q$ of $k$ centroids aimed to 
minimize the clustering cost $V(Q\mid X,\vecw)$. 

We present a
{\em wrapper,}  Algorithm~\ref{optsamples:alg},  which inputs a clustering problem, a clustering algorithm $\calA$, and $\epsilon>0$, and returns a set of $Q$ of $k$ centroids.   The 
wrapper computes weighted samples $(S,\vecw')$ of the input points $(X,\vecw)$
and applies $\calA$ to $S$.  It then performs some tests on the
clustering $Q$ returned by $\calA$,  based on which, it either
terminates and returns a clustering, or adaptively increases the sample size.  The wrapper provides
a statistical guarantee that the quality of the clustering $Q$ returned by $\calA$ on the sample $(S,\vecw')$ reflects, within $(1+\epsilon)$, its
quality on the data.

The first part of the wrapper is similar to our clustering oracle Algorithm~\ref{oracle:alg}.   
We perform $2k$ iterations of {\sc kmeans++} tor $(X,\vecw)$ to compute a
list $\{m_i\}$ of centroids and respective clustering costs
$v_i=V(\{m_1,\ldots,m_i\}\mid X,\vecw)$.  While performing this computation, we
identify a 
sweet-spot $M=\{m_1,\ldots,m_i\}$ using the coarse 
estimate~\eqref{sizeapprox:eq} of sample sizes and retain $\boldpi:X$,
which are the one2all base probabilities for $M$.  Our wrapper separately maintains
a size parameter $r$, that is initially set to  $r =  V_i/v_{2k}$.  
 From Theorem~\ref{MtoMO:thm}, 
the probabilities $r*\boldpi$ upper bound the base pps 
probabilities for all $Q$ with clustering cost $V(Q \mid (X,\vecw)) 
\geq V_M/r$.    Initially,   $r*\boldpi$ is set for cost above $v_{2k}$.
We then selects a 
fixed randomization $\vecu$, that will allow for coordination of samples selected with 
different size parameters.  

 The main iteration computes a weighted sample
 $(S,\vecw')$ selected with probabilities $\epsilon^{-2}r * \pi$.  
Our algorithm is applied to the sample $Q\gets \calA(S,\vecw')$  to 
obtain a set $Q$ of $k$ centroids.  We compute (or
estimate from a validation sample) the clustering cost over the full
dataset $V_Q = V(Q
\mid x,\vecw)$.  If $V_Q$ is not lower than $V_M/r$ and is also not
much higher than the sample clustering cost $V(Q \mid S,\vecw')$, we
break and return the best $Q$ we found so far. 
 Otherwise, we increase the size parameter $r$, augment the
sample accordingly, and iterate.  The increase in the size parameter
at least doubles it and is set so that (i)~We have $V_M/r \leq \overline{V}$, where
$\overline{V}$ is the smallest clustering cost encountered so far.  (ii)~
The set $Q$ that was underestimated by the sample has estimate
that is high enough to clear it from $\overline{V}$ or to comprise an
accurate estimate.

\subsection{Analysis}
We show that if our algorithm $\calA$ provides a certain approximation
ratio on the quality of the clustering, then this ratio would also
hold (approximately, with high confidence) over the full data set.  A
similar argument applies to a bicriteria bound.

  The wrapper works with an optimistic initial choice of $r$, but increases it
  adaptively as necessary.   The basis of the correctness of our
  algorithm is that we are able to detect when our choice of $r$ is
  too low.  

 There are two separate issues that we tackle with adaptivity instead
 of with a pessimistic worst-case bound.
  The first is also addressed by our  feedback oracle:  
For accurate 
 estimates we need $V_M/r$ to be lower than $V^* = V(Q^* \mid
 X,\vecw)$  (the optimal clustering
 cost over $X$), which we do not know.  Initially, $V_M/r=v_{2k}$,  which may
 be higher than $V^*$.   We  increase $r$
when we find a clustering $Q$ with $V(Q) < V_M/r$.  
The potential ``bad''  event is  when the optimum clustering $Q^*$ has 
$V^* \ll V_M/r$ but is overestimated by a large amount in the sample
resulting in the  sample optimum $V^*_S$ is much larger than $V^*$.
As a consequence, 
the clustering algorithm $\calA$ applied to the sample can find 
$Q$, for which the estimate is correct, and has cost above $V_m/r$. 
The approximation ratio over the sample is $V(Q \mid S)/V^*_S$ which
can be much better than the true (much weaker) approximation ratio 
$V(Q \mid S)/V^*$ over the full data.   

  This bad event happens when $V^*_S \gg V^*$.  But note that in
  expectation, 
 $\E[V^*_S] \leq V^*$.   Moreover, the probability of this bad event
is bounded by $\exp(-\epsilon^{-2}/6)$ (see Theorem~\ref{ppsdedicated:thm} and Corollary~\ref{weakpps:lemma}).
We can make the probability of such bad event smaller by augmenting the wrapper as
follows.  When the wrapper is ready to return $Q$, we generate multiple samples of the same size and apply
$\calA$ to all these samples and take the best clustering
generated. If we find a clustering with cost below $V_m/r$, we
continue the algorithm.  Otherwise, we return the best $Q$.   The
probability that all the repetitions are ``bad''  drops exponentially with the number of
repetitions (samples) we use.

The second issue is inherent with optimization over samples.
Suppose now that  $r$ is such that
$V^*\geq V_M/r$.  The statistical guarantees provided by the sample are ``\eachg,'' 
which assure us that the cost is estimated well for a {\em given}
$Q$.   
In particular,  $V(Q^* \mid S,\vecw')$  is well concentrated 
around $V^*$  (Theorem~\ref{ppsdedicated:thm}). 
This means that $V^*_S$, the optimal clustering cost over $S$,  can only 
(essentially - up to concentration) be lower than $V^*$.  

When we consider all $Q$ of size $k$, potentially an infinite or a
very large number of them, it is
possible that some $Q$ has clustering cost  $V(Q\mid X,\vecw)
\gg V^*$  but is grossly underestimated in the sample,  having sample-based cost
$V(Q \mid S, \{w_x/p_x\}) < V^*$.  In this 
case, $V^*_S \ll V^*$ and our algorithm $\calA$ that is applied to the sample will be
fooled and can return such a 
$Q$.   The worst-case approach to this issue is to use a union or a
dimensionality bound that drastically increases sample size.   We get around it using an adaptive optimization framework
\cite{multiobjective:2015}. 

  We can identify and handle this scenario, however, by
testing $Q$ returned by the base algorithm to determine if our
algorithm was ``fooled'' by the sample:  
\begin{equation} \label{test:eq}
 V(Q \mid X,\vecw) \leq (1+\epsilon)  V(Q \mid S, \{w_x/p_x\}) \ . 
\end{equation}
by either computing the exact cost $V(Q \mid 
X,\vecw)$ or by drawing another 
independent  {\em validation} sample $S'$, and using the estimate 
$V(Q \mid S, \{w_x/p_x\})$. 
When the test fails, we increase the sample size and repeat.  In fact,
we at least double the sample size parameter, but otherwise increase it at least
to the point that  $V(Q \mid S, \{w_x/p_x\})$ can no longer fool the
algorithm.  The only bad event possible here is that the sample
optimum is much larger than $V(Q^*)$.   But as noted, when $V^*\geq
V_M/r$ the probability of this for a particular sample is bounded by Theorem~\ref{ppsdedicated:thm}.
Moreover, note
that each increase of the sample size significantly strengthens the
concentration of estimates for particular $Q$.   Thus, the worst quality, over
iterations, in which $Q^*$ is estimated in the sample is dominated by the
first iteration with $V(Q^*) \geq V_M/r$.    Therefore,   the
approximation ratio over the sample is at least (up to the statistical
concentration of the estimates of $Q^*$) the ratio over the full data.

\subsection{Computation}
The computation performed is dominated by two components.  The first is the $2k$ iterations of {\sc kmeans++} on the data, which are dominated by $2k|X|$ pairwise distance computations. These is the only component that must be performed over the original data.
The second is the application of $\calA$ to the sample.  When $\calA$ is (super)linear, it is dominated by the largest sample we use.

 Note that correctness does not depend on using $2k$ iterations.  We
 can apply one2all to any set $M$.  The only catch is that we may end
 up using a very large value of $r$ and a larger sample size.
An added optimization, which is not in the pseudocode, is to 
perform the {\sc kmeans++} iterations more sparingly.  Balancing
the size of the sample (the final parameter $r$ and its product with
$M$)  and the computation cost of 
additional iterations over the full data.   


\section{Experiments} \label{exper:sec}

  We performed illustrative  experiments for Euclidean
  $k$-means clustering on both synthetic and real-world data sets.
  We implemented our wrapper Algorithm~\ref{optsamples:alg} in numpy with the following
  base clustering algorithm $\calA$:  We use $5$ applications of  {\sc
  kmeans++} and take the set of $k$ centroids that has the smallest clustering cost.  This set is used as an initialization to
  20 iterations of Lloyd's algorithm.  The use of  {\sc kmeans++} to
initialize Lloyd's algorithm is a prevalent method in practice.

\paragraph{Synthetic data:}
We generated synthetic data sets by drawing $n$ points 
$X \subset R^d$ from a mixture of  $k$  Gaussians.
The means of the Gaussians are arranged to lie in a line with equal 
distances.  The standard deviations of the Gaussians were 
drawn from a range equal to the
 distance to the closest mean.
As a reference, we use the means of the 
 Gaussians as the {\em ground truth} centroids.

\paragraph{MNIST and Fashion MNIST  datasets:}
 We use the MNIST data set of images of handwritten digits \cite{mnist2010} and
 the Fashion data set of images of clothing items \cite{fashion-mnist2017}. Both data sets
 contain $n=6\times 10^5$ images coded as $d=784$ dimensional
 vectors. There are $k=10$ natural classes that correspond to the 10
 digits or 10 types of clothing items. Our reference ground-truth centroids were taken as the mean of each class.  

\paragraph{Worst-case bounds:}
We also  report, for comparison, sizes based on
state-of-the-art
coresets  constructions that provide the same
statistical guarantees.   The coreset sizes are determined using worst-case  upper
bounds. When constant factors are not specified, we {\em
  underestimate} them.  
The constructions can be viewed as having two components.  The first is an
upper bound on the size of a coreset that provides an  \eachg\ guarantee.  For our purposes, we would also need
a constructive way to obtain such a coreset.  The second is an upper  bound on the increase factor that is needed to
obtain an \allg\ guarantee.
We make here gross underestimates of worst-case coreset sizes.  The best bound on an
\eachg\ coreset size is slightly underestimated by
$8 \rho^2 k \epsilon^{-2} = 32 k \epsilon^{-2}$.  
For an actual construction, we can use {\sc kmeans++} and the tightest
worst-case bounds on the bicriteria approximation quality it provides.
The state of the art \cite{Wei:NIPS2016} is that with
$\beta k$ centroids we are (in expectation) within a factor of
$8(1+\varphi/(\beta-1))$ of the
optimal clustering cost (for Euclidean metric),
where $\varphi\approx 1.618$ is the golden
ratio.  There are no concentration results, and Markov inequality is used to obtain confidence bounds:  There is 50\% probability of being below twice
this expectation, which is $16(1+\varphi/(\beta-1))$.

The sample size depends both on the number of centroids we use and on the approximation quality: We need to minimize the product of $\beta$ and the approximation factor.  The expression $\beta(1+\varphi/(\beta-1))$ for $\beta>1$ is minimized at  $\beta\approx 1.53$ and the factor is $\approx 6.1$.  So we obtain an increase factor on sample size that is at least $6\times 16=96$.

Combining this all, we get an underestimate of $96 * 32 * k\epsilon^{-2}\approx 3000 k\epsilon^{-2}$ for that component.
We then consider the bound on the increase factor.
The state of the art bounds \cite{BravermanFL:arxiv2016}, based on
union bound and dimensionality arguments, are
 $O(\min(n,d/\epsilon))$ for Euclidean  and $O(\log k \log n)$ 
for general metric spaces. The hidden constant factors are not
specified and we underestimate them here to be  equal to $1$.
Combining, we underestimate the
best worst-case bound on coreset size by
$$\min\{n, 3000 k \epsilon^{-2} \min\{\log k \log n, \min(n,d/\epsilon)\}\ .$$

\paragraph{Adaptive bounds:}
Table~\ref{results:tab} reports the results of our experiments.
The first four columns report the basic parameters of each data set:
The number of points $n$, clusters $k$, dimension $d$, and the specified value of $\epsilon$ for the desired statistical
guarantee.  The middle columns report the
final sample size $|S|$ used by the algorithm
as a fraction of $n$, an underestimate on the corresponding coreset
size from state of the art worst-case bounds,  and the
gain factor in sample size by using our adaptive algorithm instead
of a worst-case bound.   We can observe significant benefit that increases with the size of the data sets.  On the  MNIST data, the worst-case approach provides no data reduction.

 The third  set of columns reports the accuracy of the sample-based estimate of the cost of the final clustering $Q$.  We can see that the error 
is very small (much smaller than $\epsilon$). We also report 
the quality of the final clustering $Q$ and the quality of the
 clusters obtained by applying {\sc kmeans++} to $X$, relative to the
 cost of the  ``ground truth'' centroids.   We can see that the cost
 of the final clustering is very close (in the case of MNIST, is
 lower) than the ``ground truth'' cost.  We also observe significant
 improvement over the cost of the {\sc kmeans++} centroids used for
 initialization.

 The last column reports the number of {\sc kmeans++} iterations on
 the full data set that was eventually used (the sweet spot value).
 This sweet spot optimizes for the overhead per sample size. This
 means that in effect fewer than $k$ kmeans++ iterations over the full
 data were used.

 Finally, we take a closer look at the benefit of the number of 
iterations of {\sc kmeans++}  that are performed on the full data
 set and used as input to one2all.
Figure~\ref{pluspluscosts:fig} shows properties for the sequence of
centroids $\{m_i\}$ returned by the {\sc kmeans++} algorithms on the
mnist, fashion-mnist, and one of the mixture synthetics datasets with
$k=20$ natural clusters.   The first is the clustering cost of each
prefix, divided by the cost of the ground truth clustering.    We can
see that on our synthetic data with spread out clusters there is
significant cost reduction with the first few iterations whereas with the
two natural data sets, the cost of the first (random) centroids is
within a factor 3 of the ground-truth cost with 10 centroids.
The second plot shows the  sample size ``overhead factor'' when we use
one2all on a prefix $\{m_1,\ldots,m_i\}$ of the {\sc kmeans++} centroids to obtain 
\eachg\ guarantees for $Q$ with clustering costs that are at least the
ground-truth cost.  To do so, we apply one2all to the set with $\alpha
= V(\{m_1,\ldots,m_i\}\mid X)/V_{\text{ground-truth}}$.  The
resulting overhead is proportional to 
$$\alpha i = \frac{V(\{m_1,\ldots,m_i\}\mid
  X)}{V_{\text{ground-truth}}} i\ .$$
We can see that with our two natural data sets, the sweet spot is
obtained with the first centroid.  Moreover, most of the benefit is
already obtained after 5 centroids.

  For the task of providing an efficient clustering costs oracle, we
would like to also optimize the final sample size.   We can see that
the sweet-spot choice of the number of centroids provides significant
benefit (order of magnitude reduction on the natural data set) compared to the worst-case choice (of using $\approx 1.5 k$ centroids).

  For clustering, the figure provides indication for the benefit of
  additional optimization which incorporates the cost versus benefit of additional
{\sc kmeans++} iterations that are performed on the full data set $X$.
As mentioned, we can adaptively perform additional iterations as to
balance its cost with the computation and accuracy we have on a sample
that is large enough to meet \eachg\ for $k$-clusters:  Even on data
sets where the sweet-spot required more iterations, a prohibitive cost
of performing them on the full data set may justify working with a
larger sample.

\begin{figure*}
\center 
\includegraphics[width=0.45\textwidth]{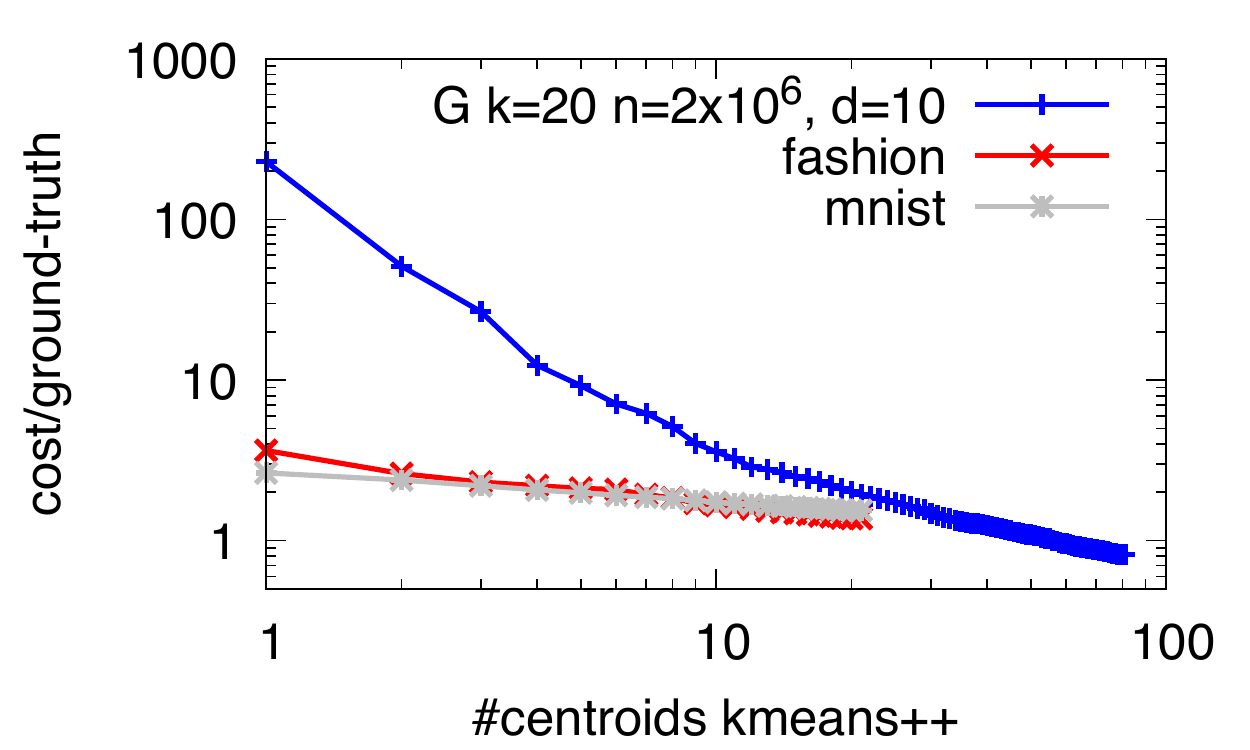}
\includegraphics[width=0.45\textwidth]{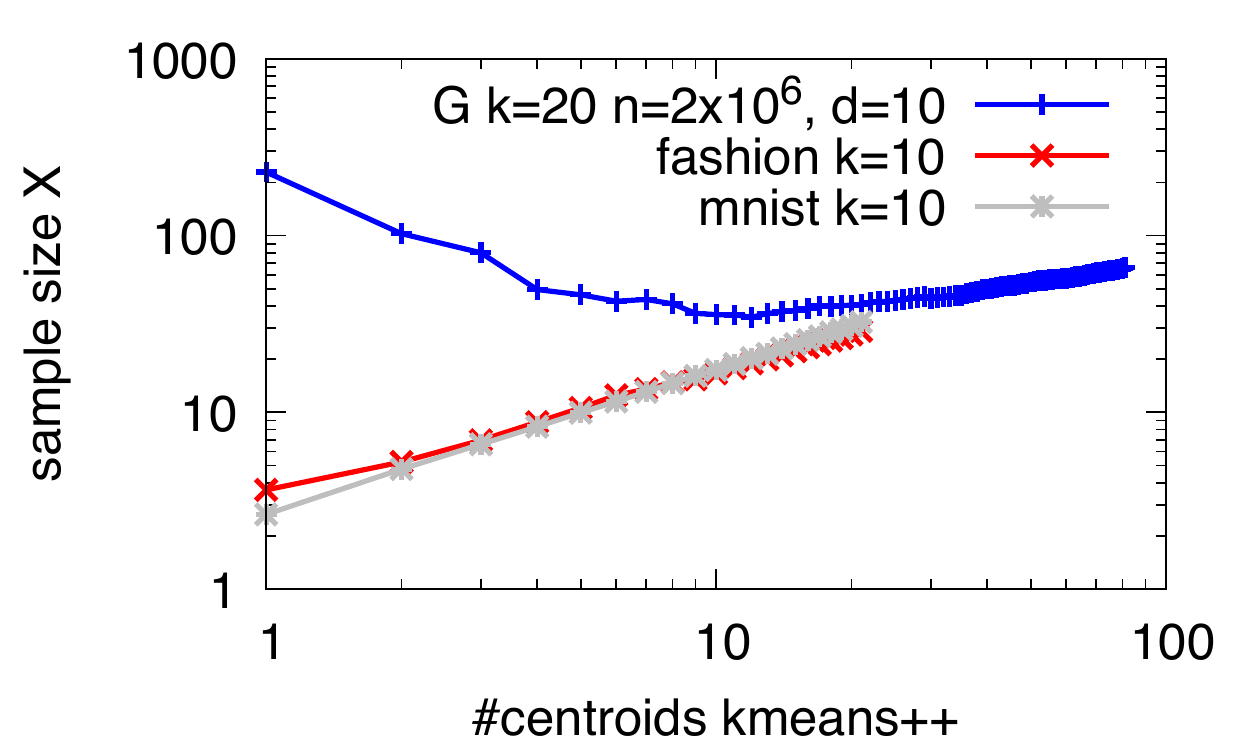}
\caption{Left: Clustering cost (divided by ground-truth clustering cost) for
  first $i$ centroids found by {\sc kmeans++} on full data set
  (averaged over 10 iterations).  Right: Sample size overhead factor when
  applying one2all to $\{m_1,\ldots,m_i\}$ to provide \eachg\
  guarantees for clustering costs that are at least the ground truth
  cost. \label{pluspluscosts:fig}}
\end{figure*}

\section{Conclusion}

We consider here clustering of a large set of points $X$ in a
(relaxed) metric space.
We present a clustering cost oracle,
 which estimates the clustering cost of an input set of centroids from
 a small set of sampled points, and
a clustering wrapper that inputs a base clustering algorithm which
 it applies to small sets of sampled points.  
At the heart of our design are our  {\em one2all}  base 
probabilities that are asigned to the points $X$.  These probabilities
are defined with respect to a set $M$ of 
centroids but yet,  a sample of size  $O(\alpha^{-1} |M|)$ (for 
  $\alpha\geq 1$) allows us to estimate the clustering cost 
of {\em any} set of centroids $Q$ with cost that is at least
$V(M)/\alpha$.  

 Our clustering cost oracle and wrapper work with weighted 
 samples taken using the {\em one2all}  probabilities for 
a set $M$ of centroids computing using the kmeans++
  algorithm.  Our oracle adaptively increase the sample size 
  (effectively increasing ``$\alpha$'') when encountering  $Q$ with cost lower than $V(M)/\alpha$. 
 Our wrapper increases the sample size when the clustering returned by 
 the base algorithm has sample cost that is either below $V(M)/\alpha$
 or does not match the cost over the full data (invoking a method of 
 adaptive optimization over samples~\cite{multiobjective:2015}).

A salient feature of
 our oracle and clustering wrapper  methods is that we
start with an optimistic small sample and increase it {\em adaptively}
only in the face of hard evidence that a larger sample is indeed
necessary for meeting the specified statistical 
 guarantees on quality.  Previous constructions use  {\em
   worst-case} size summary structures that can be much larger.
  We demonstrate experimentally the very large potential gain, of 
  orders of magnitude in sample sizes, when using our adaptive 
versus worst-case methods.

Beyond estimation and optimization of
 clustering cost, the set of distances of each $Q$  to the
one2all sample $S$ is essentially  a sketch of the full (weighted)
distance vector of $Q$ to $X$   \cite{multiw:VLDB2009}.
  Sketches of different sets $Q$ allow us to estimate relations
  between the respective full vectors, such as distance norms,
  weighted Jaccard similarity,  quantile aggregates, and more, which
  can be useful building blocks in other applications.

\notinproc{
Moreover, Euclidean
  $k$-means clustering is  a constrained rank-$k$  approximation
  problem, and  this connection facilitated interesting feedback between
  techniques designed for low-rank approximation and for
  clustering~\cite{CEM2P:STOC2015}. We thus hope that our methods and the general
method of optimization over multi-objective sample~\cite{multiobjective:2015} might lead to further
  progress on other low-rank approximation problems.
}

 \bibliographystyle{plain}
\bibliography{cycle} 

\end{document}